\documentclass[12pt]{article}

\usepackage[margin=1in]{geometry}
\usepackage{amsmath,amssymb,amsthm}
\usepackage{booktabs}
\usepackage{graphicx}
\usepackage{algorithm}
\usepackage{algpseudocode}
\usepackage{hyperref}
\usepackage{xcolor}
\usepackage{caption}
\usepackage{subcaption}
\usepackage{natbib}
\usepackage[nopatch=footnote]{microtype}
\usepackage{setspace}
\usepackage{listings}
\usepackage{tcolorbox}
\usepackage{array}
\usepackage{multirow}
\usepackage{tabularx}
\usepackage{longtable}

\newtheorem{proposition}{Proposition}
\newtheorem{lemma}{Lemma}

\newtheorem*{proposition*}{Proposition}
\newtheorem*{lemma*}{Lemma}
\theoremstyle{definition}
\newtheorem{definition}{Definition}
\theoremstyle{remark}

\hypersetup{
  colorlinks=true,
  linkcolor=blue!70!black,
  citecolor=blue!70!black,
  urlcolor=blue!70!black,
}

\title{LambdaRankIC: Directly Optimizing Rank IC for Financial Prediction}

\author{
Yan LIN\\
University of Macau\\
\texttt{yanlin@um.edu.mo}\\
Yihong SU\\
University of Macau\\
\texttt{yc47017@um.edu.mo}\\
Yi YANG\\
Hong Kong University of Science and Technology\\
\texttt{imyiyang@ust.hk}
}

\date{\today}

\begin{document}
\onehalfspacing

\maketitle
\clearpage
\begin{abstract}
In financial predictions, the performance of machine learning models is often assessed by Rank IC, which is the Spearman rank correlation between the model predictions and the realized asset returns. 
Despite its wide adoption, most existing models are trained using regression losses or ranking objectives that may not align with Rank~IC.
We propose \textbf{LambdaRankIC}, a novel learning-to-rank approach that directly optimizes Rank~IC. We circumvent the non-differentiability of the ranking operator by deriving the closed-form expression for the lambda gradients induced by the pairwise rank swaps, which enables efficient gradient-based optimization within the LambdaRank framework. We implement LambdaRankIC as a custom objective in XGBoost. Theoretically, we show that our approach optimizes an upper bound on Rank~IC. 
We evaluate the proposed approach on both simulated and real-world financial data. In simulation studies, LambdaRankIC accurately recovers the true ranking structure in noiseless settings and consistently outperforms regression-based and NDCG-oriented ranking methods under low signal-to-noise ratios and heavy-tailed noise regimes. In empirical experiments using real market data, LambdaRankIC achieves the best out-of-sample performance on evaluation metrics commonly used in finance, including Rank~IC, ICIR, monthly return, and Sharpe ratio. These results show that directly optimizing Rank~IC can yield substantial improvements over conventional learning objectives in financial predictions when the full-order ranking quality is the primary goal.

\end{abstract}
\textbf{Keywords:} Rank~IC; learning to rank; gradient boosted trees; Spearman rank correlation; return prediction;

\bigskip

\section{Introduction}
\label{sec:intro}
Quantitative approaches to investment constitute a cornerstone of modern finance. Early foundational contributions include \cite{markowitz1954portfolio}'s mean-variance framework for portfolio selection and the multi-factor asset pricing models of \cite{fama1993common}. 
More recently, advances in machine learning have prompted both researchers and practitioners to incorporate these data-driven methods into the design of investment strategies \citep{gu2020empirical}. The standard paradigm proceeds in two stages. First, a predictive model is trained to forecast future stock returns. The model can be a linear model or a more complex machine learning model (or the return model). Then in the second stage, the resulting forecast scores are used as signals in downstream tasks such as portfolio construction, factor combination, or risk model estimation \citep{grinold2000active}.

However, given the extremely low signal-to-noise ratio in financial markets, the predictive performance of the return model remains limited, even with the most advanced machine learning algorithms. \citet{gu2020empirical} report monthly out-of-sample $R^2$ values on the order of 0.5\% for their best-performing models such as gradient-boosted decision trees and neural networks, underscoring the difficulty of predicting returns in levels.
Given this low predictive power in absolute terms, 
practitioners and researchers have long recognized that the relative ranking of predicted returns carries far greater economic value than the point forecasts themselves \citep{grinold2000active}. 
For example, in cross-sectional applications, factors are commonly implemented through sorting stocks into decile portfolios based on the factor score and forming a top-minus-bottom long--short portfolio. 
The standard metric for evaluating model performance is the Rank Information Coefficient (Rank IC), defined as the Spearman rank correlation ($\rho$) between predicted and actual returns. Unlike $R^2$, which measures the accuracy of level predictions, Rank~IC captures the model's ability to correctly order stocks, which is operationally relevant for most investment applications. A model that reliably places outperformers above underperformers is valuable for portfolio construction, even if its point estimates of return magnitudes are imprecise.

Hence, there is a misalignment between the training objective and the actual objective that investors wish to optimize for when implementing investment strategies. The prevailing approach in the machine learning for asset pricing literature is to train models by minimizing mean squared error (MSE) of return predictions, yet evaluate them using Rank IC (or related, ranking-based metrics, such as decile spread returns). This train–evaluation mismatch raises a natural question: Can we instead train models to directly optimize Rank~IC?

Directly optimizing Rank IC poses a significant challenge because Spearman's rank correlation is a non-differentiable function of model scores. The ranking operator that maps continuous scores to discrete ranks has zero gradients almost everywhere. This rules out standard gradient-based optimization, including backpropagation in neural networks. One might hope that existing learning-to-rank methods for information retrieval could serve as suitable surrogates for Rank~IC optimization. However, a careful examination reveals that none of the established ranking metrics and objectives align with Rank~IC. Pairwise methods such as RankNet \citep{burges2005learning} optimize a loss closely related to Kendall's $\tau$, which treats all pairwise misordering errors as equally costly regardless of the magnitude of rank displacement. This uniform pairwise weighting is misaligned with Spearman's $\rho$, which penalizes larger rank displacements more heavily. Listwise methods that optimize Normalized Discounted Cumulative Gain (NDCG) \citep{burges2006ranknet,burges2010ranknet} employ a position discount that forces the model to concentrate learning effort on the top of the ranking.
While this approach might be well-suited for long-short portfolio construction \citep{poh2021building,lin2026empirical}, it is insensitive to the quality of ordering in the middle. Other information retrieval metrics such as MAP and ERR are designed for binary relevance ranking and top-heavy ranking scenarios, making them fundamentally incompatible with the continuous, symmetric, full-ranking nature of Rank IC.

In this paper, we propose \textbf{LambdaRankIC}, a novel learning-to-rank objective specifically designed to optimize Rank~IC. 
The core of our approach is the derivation of a simple closed-form expression for the lambda gradient of Rank~IC within the LambdaRank framework. This expression quantifies the exact change in Rank~IC induced by swapping the predicted ranks of any two arbitrary assets. These lambda gradients can then dynamically guide the optimization process toward the true Rank~IC objective, thereby effectively bypassing the zero-gradient limitation. Theoretically, we show that our LambdaRankIC approach optimizes an upper bound on the true Rank~IC.

We implement \textbf{LambdaRankIC} as a custom objective natively integrated into the XGBoost library \citep{wu2010adapting}. \footnote{Although this paper focuses on tree-based models, the proposed objective can also be extended to neural networks.We focus on tree-based models because they are the workhorse of modern financial applications and remain among the best-performing approaches for tabular data, achieving performance comparable to that of neural networks and transformer-based models \citep{gorishniy2021revisiting,bryzgalova2025forest, cong2025growing}.} 
We conduct extensive evaluations and benchmarks of the proposed approach on both simulated and real-world data. Using simulated data, we first show that, in the noise-free setting, our approach can effectively recover the ranking structure of the underlying data-generating process. In noisy settings, we demonstrate that, compared with both regression-based approaches and existing learning-to-rank objectives designed to optimize NDCG, our proposed approach achieves higher Rank~IC on the test set and is more robust to overfitting. 
Using monthly stock data, we benchmark our approach against two broad classes of existing methods: (1) approaches that first predict returns and then construct portfolios, and (2) approaches that directly optimize ranking metrics for stock returns using existing learning-to-rank objectives. We show that our approach achieves the highest out-of-sample Rank~IC and ICIR among all competing models. In terms of portfolio performance, our approach delivers significantly higher returns and a higher Sharpe ratio.

The remainder of this paper is organized as follows. Section \ref{sec:background} reviews related work on machine learning for asset pricing and learning-to-rank methods. Section \ref{sec:formulation} formulates the problem and Section \ref{sec:method} presents the derivation of LambdaRankIC objective, its implementation, as well as the theoretical justification. Section \ref{sec:simulation} reports the results from our simulations. Section \ref{sec:experiments} reports the benchmark results using real financial data. 
Section \ref{sec:conclusion} concludes and outlines directions for future research.

\section{Background and Related Work}
\label{sec:background}

\subsection{Machine Learning in Finance}

Machine learning in Finance builds on the broader evidence that asset returns are predictable from characteristics \citep{fama1993common,harvey2016man}. Empirical evidence documents that nonlinear models can improve model performance relative to linear benchmarks. For example, in a large-scale comparison, \citet{gu2020empirical} show that tree-based and neural network-based models are consistently competitive for cross-sectional return prediction. 
\citet{leippold2022machine} report similar results in international markets.
\citet{feng2024deep} and \citet{gu2021autoencoder} demonstrate performance gains from nonlinear representation learning. Researchers have also explored the applications of more advanced approaches, such as Transformer-based models \citep{kelly2025artificial}, graph neural networks \citep{tian2022inductive}, generative adversarial networks \citep{chen2024deep}, and deep reinforcement learning \citep{cong2021alphaportfolio}. Among these models, tree models have gained much attention due to their vast success in tabular datasets \citep{gorishniy2021revisiting} and efficiency in model training. For example,  
\citet{cong2025growing} adapt the tree split criterion using financial metrics, while \citet{bryzgalova2025forest} use trees to group a series of stocks to span SDF. 

Machine learning models have also been used for risk management in finance, particularly for predicting stock price volatility. For example, \cite{yang2022analyzing} propose a knowledge-driven approach leveraging textual data in financial reports for volatility prediction. \cite{zhang2024let} design a neural network with a long- and short-term memory retrieval architecture to model market volatility. \cite{he2025divide} leverage the unique structure in earnings conference call transcripts and propose a novel divide-and-contrast machine learning method for predicting firm risks. 

In the literature, the standard empirical workflow for modeling and investing remains largely decoupled. Models are typically trained to minimize a generic prediction loss (for example, Mean Squared Error of the stock return). Then, the resulting forecasts are subsequently mapped into portfolio rankings or allocation rules. Recent studies suggest that such a predict–then–rank design can be suboptimal when the ultimate objective is cross-sectional ordering rather than point prediction. In this spirit, \citet{butler2023integrating}, \citet{cong2021alphaportfolio}, and \citet{costa2023distributionally} propose end-to-end frameworks that more tightly couple prediction and portfolio choice. Our paper echoes their call for a better alignment between the training objective and investment by directly optimizing Rank~IC, a metric that is directly related to asset ranking.  

\subsection{Learning to Rank and Its Applications to Finance}
Learning-to-Rank algorithms (LTR), originally developed for information retrieval and recommender systems, have been increasingly recognized for their potential in financial applications.
LTR algorithms are generally categorized into three paradigms \citep{liu2009learning}: pointwise, pairwise, and listwise. Pointwise approaches treat each item independently, making the ranking task to standard regression or classification. Pairwise methods instead model relative preferences between item pairs, optimizing the ordering of pairwise comparisons \citep{burges2005learning, freund2003efficient}. Listwise approaches directly optimize a list-level objective that evaluates the entire ranked list at once \citep{cao2007learning, xia2008listwise}.

These algorithms are typically evaluated using ranking-specific metrics. The most widely adopted include Normalized Discounted Cumulative Gain (NDCG), Mean Average Precision (MAP), and Expected Reciprocal Rank (ERR). Among these, NDCG at position $k$ is defined as:
\begin{equation}
  \NDCG@k = \frac{\text{DCG}@k}{\text{IDCG}@k}, \quad
  \text{DCG}@k = \sum_{r=1}^{k} \frac{g(y_r)}{\log_2(1+r)},
  \label{eq:ndcg}
\end{equation}
where $g(y)$ is the gain function (typically $2^y - 1$) and IDCG is the ideal
(maximum) DCG.

Several studies have applied learning-to-rank methods to the finance domain, arguing that traditional ``regress-then-rank'' approach is sub-optimal because they minimize pointwise prediction error rather than explicitly learning the relative ordering of asset returns. For example, 
\citet{poh2021building} enhanced cross-sectional momentum strategies by applying various LTR algorithms to rank assets directly. They demonstrate that LTR algorithms improve the trading performance measured in terms of Sharpe Ratio compared to traditional predict-then-rank approaches. \citet{song2017stock} combined LTR algorithms with news sentiment indicators for stock portfolio selection, showing that RankNet and ListNet produce robust rankings under different market conditions. \citet{feng2019temporal} proposed a Temporal Relational Ranking framework that incorporates inter-stock relations through temporal graph convolutions, formulating stock prediction as a ranking task to capture time-sensitive dependencies across assets. \citet{ma2022stock} proposed a deep multi-task learning framework (MTSR) that stabilizes and guides the training of complex listwise stock ranking models by leveraging auxiliary tasks to address the training difficulties inherent in small-sample, high-dimensional financial data. \citet{zhang2022constructing} developed the ListFold algorithm, a listwise loss function that simultaneously emphasizes both the top and bottom of the ranked list to better align with the needs of long-short portfolio construction. \citet{saha2021stock} adopted a listwise approach using top-$k$ probability distributions. More recently, \citet{lin2026empirical} provides the first systematic investigation of the performance of all three learning‑to‑rank methods based on an extensive set of firm characteristics spanning three decades. The authors compare these methods with traditional regression and probabilistic models and develop a dual‑ranker framework to mitigate the algorithms' heavy emphasis on top‑ranked items.

Despite these advances, all existing LTR applications in finance inherit ranking metrics from information retrieval as their optimization targets, such as NDCG. However, NDCG is a position-discounted, top-heavy metric that concentrates signal on items near the top of the list, which is misaligned with the needs of cross-sectional asset pricing, where ranking quality across all positions in the cross-section matters. None of the aforementioned studies consider optimizing Rank IC, which is the standard evaluation metric used by quantitative portfolio managers and factor investors. This gap motivates our development of LambdaRankIC.

\section{Problem Formulation}
\label{sec:formulation}
\subsection{Setup and Notation}
We formulate our problem as ranking stocks in the cross-section (e.g., monthly), similar to \citep{gu2020empirical}. In practice, this is analogous to predicting the next-period stock returns (or the ranking of stock returns).
Periodically, investors build a model to generate investment signals. This signal is then used in the downstream task to form investment strategies. One naive strategy would be going long (buying) assets with the highest scores and shorting (selling) assets with the lowest scores, subject to some constraints \citep{coqueret2020machine}. While we formulate the problem as a cross-sectional prediction problem, it can be easily extended to incorporate time-series information or time-series prediction, as long as the problem involves ranking a group of assets.

Let $\mathcal{G} = \{G_1, \ldots, G_m, \ldots , G_M\}$ denote $M$ groups (e.g., months), where
group $G_m$ contains $n_m$ items (e.g., the universe of stocks in a month).\footnote{In the learning to rank literature, this is called a query.}
Each item $i \in G_m$ is associated with a feature vector $\mathbf{x}_{im} \in \mathbb{R}^d$
and a real-valued label $y_{im} \in \mathbb{R}$ representing its future return.\footnote{This corresponds to the relevance score in the learning to rank literature. In the factor synthesis literature, $f$ is sometimes called the \emph{composite
signal} or \emph{alpha score} \citep{tulchinsky2019finding}; here we treat it as a generic scoring function learned by the model.}
A parameterized model $f_\theta:\mathbb{R}^d \to \mathbb{R}$ assigns a score $s_{im} = f_\theta(\mathbf{x}_{im})$
to each item. 
\paragraph{Learning to Rank}
Collecting the scores $s_{im}$ in month $m$ into a vector $\mathbf{s}_{m}$ and the corresponding labels $y_{im}$ into $\mathbf{y}_m$, and defining a loss function $l$ that takes the labels and scores and outputs a real value, the total loss is: 
\begin{equation}
  \mathcal{L}=\frac{1}{M}\sum_m{l(\mathbf{y}_m,\mathbf{s}_m)}
  \label{eq:total_loss}
\end{equation}

A learning-to-rank algorithm finds the optimal $f$ that minimizes the total loss: 
\begin{equation}
  \hat{f} = \arg\min_{\theta} \mathcal{L}
\end{equation}

\subsection{Rank IC (Spearman Rank Correlation)}

\begin{definition}[Predicted rank and label rank]
Within group $G_m$ of size $n$: \footnote{We drop the subscript $m$ for brevity.}
\begin{itemize}
  \item The \emph{predicted rank} $\hat{r}_i \in \{1,\ldots,n\}$ is the position
    of item $i$ when items are sorted by score $s_i$ in descending order
    (rank~1 = highest score).
  \item The \emph{actual rank} $\tilde{y}_i \in \{1,\ldots,n\}$ is the position
    of item $i$ when items are sorted by label $y_i$ in descending order
    (rank~1 = highest label). We assume no ties. 
    Ties are broken by the original data index.
\end{itemize}
\end{definition}

\begin{definition}[Rank IC]
For a group of $n$ assets with predicted ranks $\hat{r} = (\hat{r}_1,\ldots,\hat{r}_n)$
and label ranks $\tilde{y} = (\tilde{y}_1,\ldots,\tilde{y}_n)$, the Spearman rank
correlation~\citep{spearman1904proof} is:
\begin{equation}
  \rho 
  = \frac{\operatorname{Cov}(\hat{r},\tilde{y})}{\sigma_{\hat{r}}\,\sigma_{\tilde{y}}} 
  \label{eq:rankic}
\end{equation}
\end{definition}

Equation~\eqref{eq:rankic} is a shorthand for the Pearson correlation between the two
rank vectors when there are no ties in either ranking.
In the finance literature, $\rho$ computed cross-sectionally each period and averaged
over time is called the \emph{Rank Information Coefficient} (Rank~IC).






\section{The LambdaRankIC Objective and Its Implementations}
\label{sec:method}
\subsection{LambdaRank Framework and LambdaMART}
Before deriving the LambdaRankIC objective, we first describe the LambdaRank 
framework and LambdaMART algorithm, which form the basis of our approach.
The LambdaRank framework follows a probabilistic pairwise setup 
\citep{burges2010ranknet,burges2006ranknet}.
For a given pair of items $(i,j)$ with feature vectors $\mathbf{x}_i, \mathbf{x}_j 
\in \mathbb{R}^d$, a learning-to-rank function $f$ maps their features into scores 
$s_i = f(\mathbf{x}_i)$ and $s_j = f(\mathbf{x}_j)$. The score difference is then 
mapped to a probability via a sigmoid function:
\begin{equation}
  P_{ij} \equiv P(s_i > s_j) = \frac{1}{1 + e^{-\sigma(s_i - s_j)}},
  \label{eq:pij}
\end{equation}
where $\sigma$ is a shape parameter controlling the steepness of the sigmoid, and 
$P_{ij}$ denotes the modeled probability that item $i$ should be ranked above item $j$.
The cost function is then written as the cross-entropy between $P_{ij}$ and the 
target probability:
\begin{equation}
  C = -\bar{P}_{ij} \log P_{ij} - (1 - \bar{P}_{ij}) \log (1 - P_{ij}),
  \label{eq:original_ranknet_loss}
\end{equation}
where $\bar{P}_{ij}$ is the target (ground-truth) probability that item $i$ ranks 
above item $j$ (e.g., $\bar{P}_{ij} = 1$ if $i$ is ranked higher than $j$, 
$\bar{P}_{ij} = 0$ if $j$ is ranked higher, or $\bar{P}_{ij} = 0.5$ if tied).
The gradient with respect to $s_i$ is:
\begin{equation}
\label{lambda}
    \frac{\partial C}{\partial s_i} = \sigma(\bar{P}_{ij} - P_{ij}) = 
    -\frac{\partial C}{\partial s_j} \equiv \lambda_{ij}
\end{equation}
Equation~\ref{lambda} is the original RankNet gradient \citep{burges2005learning}, 
which can be directly optimized via gradient descent. The challenge is that there is a disconnect between the loss function in Eq.~ \eqref{eq:original_ranknet_loss} and the evaluation metrics, such as the Normalized Discounted Cumulative Gain (NDCG). Typically, these 
evaluation metrics are based on the sorted ranking of items, which are generally 
non-differentiable with respect to model parameters. \citet{burges2006ranknet} show 
that scaling $\lambda_{ij}$ by the change in the target metric when swapping items 
$i$ and $j$ provides a good approximation to the gradient of the target metric. For 
example, when the target metric is NDCG, the scaled lambda is:
\begin{equation}
    \lambda_{ij}' = \lambda_{ij} \cdot \Delta\text{NDCG}_{ij}
\end{equation}
where $\Delta\text{NDCG}_{ij}$ is the absolute change in NDCG when items $i$ and $j$ are 
swapped. This is the LambdaRank algorithm. Although originally developed for NDCG, 
\citet{donmez2009efficiently} empirically show that the approach extends to other 
non-differentiable rank metrics such as Mean Average Precision (MAP) and Mean 
Reciprocal Rank (MRR).
The original LambdaRank algorithm was developed for neural networks. Its tree-based 
counterpart, LambdaMART \citep{friedman2001greedy}, has a standard implementation in 
XGBoost \citep{chen2016xgboost}. However, none of the standard implementations 
support Rank~IC as the target metric, which motivates our work. In the following 
section, we derive $\Delta$RankIC for use in the LambdaRank framework.

\subsection{Derivation of \texorpdfstring{$\Delta$}{Delta}RankIC}

We now derive the closed-form expression for the change in Rank~IC when two items 
are swapped.

\begin{lemma}[Variance of uniform ranks]
\label{lem:var}
For $n$ consecutive integers $\{1, 2, \ldots, n\}$,
\begin{equation}
  \operatorname{Var}(1,\ldots,n) = \frac{n^2-1}{12}.
  \label{eq:var}
\end{equation}
\end{lemma}

Using Lemma~\ref{lem:var}, the Spearman rank correlation can be rewritten as the
Pearson correlation between predicted ranks $\hat{r}$ and true ranks $\tilde{y}$:

\begin{equation}
  \rho 
  = \frac{\operatorname{Cov}(\hat{r},\tilde{y})}{\sigma_{\hat{r}}\,\sigma_{\tilde{y}}} 
  = \frac{\operatorname{Cov}(\hat{r},\tilde{y})}{\operatorname{Var}(1,\ldots,n)} 
  = \frac{\frac{1}{n}\sum_{i=1}^n \hat{r}_i\tilde{y}_i - \bar{r}\,\bar{\tilde{y}}}
         {(n^2-1)/12} 
  = \frac{12\!\left(\sum_{i=1}^n \hat{r}_i\tilde{y}_i - n\bar{r}\,\bar{\tilde{y}}\right)}
         {n(n^2-1)},
  \label{eq:rho}
\end{equation}
where $\bar{r} = \bar{\tilde{y}} = (n+1)/2$, since both rank vectors are permutations
of $\{1,\ldots,n\}$.

\begin{proposition}[$\Delta$RankIC]
\label{prop:delta_rankic}
Let item $i$ have predicted rank $\hat{r}_i$ and actual rank $\tilde{y}_i$, and let 
item $j$ have predicted rank $\hat{r}_j$ and actual rank $\tilde{y}_j$. 
After swapping items $i$ and $j$ in the predicted ranking (so $i$ moves to 
$\hat{r}_j$ and $j$ moves to $\hat{r}_i$), the change in Rank~IC is:
\begin{equation}
  \boxed{
    \Delta\text{RankIC} \;=\;
    \frac{12\,|\hat{r}_j - \hat{r}_i||\tilde{y}_i - \tilde{y}_j|}{n\,(n^2-1)}.
  }
  \label{eq:delta_rankic}
\end{equation}
\end{proposition}

\begin{proof}
From Eq.~\eqref{eq:rho}, the numerator of $\rho$ depends only on 
$D = \sum_{i=1}^n \hat{r}_i\tilde{y}_i$, given a fixed group. Expanding:
\[
  D = \hat{r}_1\tilde{y}_1 + \cdots + \hat{r}_i\tilde{y}_i + \cdots + 
      \hat{r}_j\tilde{y}_j + \cdots + \hat{r}_n\tilde{y}_n.
\]
After swapping $i$ and $j$:
\[
  D' = \hat{r}_1\tilde{y}_1 + \cdots + \hat{r}_j\tilde{y}_i + \cdots + 
       \hat{r}_i\tilde{y}_j + \cdots + \hat{r}_n\tilde{y}_n.
\]
Hence:
\begin{align}
  \Delta D &= D' - D 
           = (\hat{r}_j - \hat{r}_i)\tilde{y}_i + (\hat{r}_i - \hat{r}_j)\tilde{y}_j
           = (\hat{r}_j - \hat{r}_i)(\tilde{y}_i - \tilde{y}_j).
  \label{eq:deltaD}
\end{align}
Since the denominator $n(n^2-1)/12$ and the term $n\bar{r}\,\bar{\tilde{y}}$ are 
invariant to any permutation of the predicted ranking, we have:
\begin{equation}
  \Delta\text{RankIC} = |\Delta\rho|
  = |\frac{12\,D'}{n(n^2-1)} - \frac{12\,D}{n(n^2-1)}|
  = \frac{12\,|\Delta D|}{n(n^2-1)}
  = \frac{12\,|\hat{r}_j-\hat{r}_i|\,|\tilde{y}_i-\tilde{y}_j|}{n\,(n^2-1)}.
\end{equation}
\end{proof}

Equation~\eqref{eq:delta_rankic} has an intuitive interpretation. The magnitude of improvement on Rank~IC from swapping $i$ and $j$ is determined by how far apart the two items are in the predicted ranking ($\hat{r}_j - \hat{r}_i$) and their actual rank separation ($\tilde{y}_i - \tilde{y}_j$). 

\paragraph{Comparison with LambdaRank with NDCG objective (LambdaRank-NDCG)}
In LambdaRank with NDCG as the objective, the scaling factor is:
\begin{equation}
\label{eq:ndcg_define}
  \Delta\text{NDCG}_{ij} = |\frac{2^{y_i}-2^{y_j}}{\text{IDCG}}|
  |\frac{1}{\log_2(1+\hat{r}_i)}-\frac{1}{\log_2(1+\hat{r}_j)}|.
\end{equation}
Since the NDCG discount function $1/\log_2(1+r)$ is convex and monotonically 
decreasing in $r$, the difference 
$\left(\frac{1}{\log_2(1+\hat{r}_i)}-\frac{1}{\log_2(1+\hat{r}_j)}\right)$ 
is substantially larger for swaps near the top of the ranked list than for swaps at 
lower positions. Consequently, the resulting $\lambda$ gradients are concentrated on 
top-ranked items, making LambdaRank-NDCG inherently top-heavy.

Table~\ref{tab:delta_comparison} compares the scaling factor $\Delta$ for the two 
existing LambdaRank objectives in XGBoost and our proposed objective.
$\Delta$RankIC can be computed in $O(1)$ time and is easily integrated as an 
extension of existing objectives.
Algorithm~\ref{alg:gradient_loop} shows the gradient computation loop implemented 
in XGBoost.

\begin{table}[ht]
\centering
\caption{Comparison of LambdaRankIC and Existing Approaches}
\label{tab:delta_comparison}
\resizebox{\textwidth}{!}{%
\begin{tabular}{llll}
\toprule
Objective & $\Delta$ (Scaling Factor) & Pair Weighting & Portfolio Alignment \\
\midrule
\texttt{rank:pairwise} & $1$ & Uniform & equal weight for all misordered pairs \\
\texttt{rank:ndcg}     
  & $\dfrac{|2^{y_i}-2^{y_j}|}{\text{IDCG}}
    |\dfrac{1}{\log_2(1+\hat{r}_i)}-\dfrac{1}{\log_2(1+\hat{r}_j)}|$
  & Concentrated at top ranks &  Optimizes head of ranking \\
\texttt{rank:ic} (ours)       
  & $\dfrac{12|\hat{r}_j-\hat{r}_i||\tilde{y}_i-\tilde{y}_j|}{n(n^2-1)}$
  & Proportional to rank separation & Directly optimize for Rank~IC \\
\bottomrule
\end{tabular}%
}
\end{table}

\begin{algorithm}[ht]
\caption{LambdaRank with RankIC}
\label{alg:gradient_loop}
\begin{algorithmic}[1]
\Require Predictions $\mathbf{s}$, labels $\mathbf{y}$, learning rate $\eta$, 
         features $\mathbf{x}$
\Statex
\Comment{Pre-processing}
\For{each group $g$ \textbf{in parallel}}
  \State Stable-sort elements in group $g$ by label in descending order
  \State Assign ranks $1, 2, \ldots$ in sorted order as label ranks $\tilde{\mathbf{y}}$
\EndFor
\Statex
\Comment{Boosting loop}
\For{iteration $= 1$ \textbf{to} \textsc{MaxIter}}
  \For{each group $g$ \textbf{in parallel}}
    \State $n \leftarrow$ number of elements in group $g$
    \State Sort elements in group $g$ by prediction score $\mathbf{s}$ in descending 
           order
    \State Obtain predicted ranks $\hat{\mathbf{r}}$ from the sorted order
    \State Initialize $\lambda_i \leftarrow 0,\; h_i \leftarrow 0$ for all $i$ in 
           group $g$
    \For{each pair $(i,j)$ sampled from $g$}
      \If{$\tilde{y}_i < \tilde{y}_j$}
        \State swap$(i,j)$
        \Comment{Ensure $\tilde{y}_i \geq \tilde{y}_j$}
      \EndIf
      \State $\Delta \leftarrow \dfrac{12\,(\hat{r}_j - \hat{r}_i)\,
             (\tilde{y}_i - \tilde{y}_j)}{n(n^2-1)}$
      \State $p \leftarrow \sigma(s_i - s_j)$
      \State $\lambda_{ij} \leftarrow (p - 1)\cdot|\Delta|$
      \State $h_{ij} \leftarrow 2\,p\,(1-p)\cdot|\Delta|$
      \State $\lambda_i \mathrel{+}= \lambda_{ij}$;\quad $\lambda_j \mathrel{-}= 
             \lambda_{ij}$
      \State $h_i \mathrel{+}= h_{ij}$;\quad $h_j \mathrel{+}= h_{ij}$
    \EndFor
  \EndFor
  \State Fit tree $T$ using $\{(\mathbf{x}_i,\,\lambda_i,\,h_i)\}$ for all $i$ \Comment{Tree fitting and score update}
  \State $s_i \leftarrow s_i + \eta\cdot T(\mathbf{x}_i)$ for all $i$
\EndFor
\end{algorithmic}
\end{algorithm}

\subsection{Theoretical Justification of LambdaRankIC}
One of the limitations of LambdaRank is that, although it works empirically well  \citep{donmez2009efficiently}, there is no guarantee that weighting $\lambda$ with the change of ranking metric (e.g., $\Delta \text{RankIC}$) from a pairwise swap can optimize the global ranking metrics. In this section, we provide some theoretical justification of LambdaRankIC based on the LambdaLoss framework \citep{wang2018lambdaloss}. 
The paper establishes that LambdaRank is a special configuration of the LambdaLoss framework. More importantly, LambdaRank with $\Delta \mathrm{NDCG}$ weighting optimizes a coarse upper bound on global NDCG. In what follows, we show that, by weighting $\lambda$ with $\Delta \mathrm{RankIC}$, the LambdaRankIC algorithm optimizes an upper bound on $1 - \rho$, so that minimizing the LambdaRankIC loss indirectly maximizes Rank~IC. We first introduce the LambdaLoss framework and establish how LambdaRankIC relates to it. We then show that the LambdaRankIC loss is an upper bound on $1 - \rho$.
\subsubsection{The LambdaLoss Framework}

Following \cite{wang2018lambdaloss}, in Eq.~\eqref{eq:total_loss}, $l(\mathbf{y}_m,\mathbf{s}_m)$ can be written as
\begin{equation}
    l(\mathbf{y},\mathbf{s})=-\log_2 P(\mathbf{y}|\mathbf{s})=-\log_2\sum_{\pi \in \Pi}P(\mathbf{y}|\mathbf{s},\pi)P(\pi|\mathbf{s})
    \label{eq:lambdaloss}
\end{equation}
where $\pi$ is a (latent) ranked list and $P(\pi|\mathbf{s})$ is the distribution conditional on $\mathbf{s}$. We drop the subscript $m$ and use $\mathbf{y}$ and $\mathbf{s}$ for brevity. 
Eq.~\eqref{eq:lambdaloss} is the loss function in the LambdaLoss framework. While Eq.~\eqref{eq:lambdaloss} is not convex in general, it can be minimized by the expectation maximization algorithm \citep{dempster1977maximum}. Essentially, the procedure starts with a random guess of the model parameter $\theta$ and iterates over the E-step and M-step: at the E-step, the distribution of the latent variable $\pi$ is estimated based on the current model $f^{(t)}$; at the M-step, the model parameters are updated by minimizing the expected loss. 
\cite{wang2018lambdaloss} shows that LambdaRank is a special case under the LambdaLoss framework with two choices for $P(\mathbf{y}\mid\mathbf{s},\pi)$ and $P(\pi\mid\mathbf{s})$: 

\paragraph{Hard assignment}  
LambdaRank adopts hard assignment distribution $H(\pi|\mathbf{s})$ to replace $P(\pi|\mathbf{s})$. $H(\pi|\mathbf{s})$ assigns 
probability one to the ranked list $\hat{\pi}$ obtained by sorting items by decreasing scores $\mathbf{s}$, and zero to all other permutations.
In this case,
\begin{equation}
l(\mathbf{y},\mathbf{s})=-\log_2 P(\mathbf{y}|\mathbf{s},\hat{\pi})
\end{equation}

\paragraph{Generalized Bradley--Terry Likelihood}
For items (assets) \(i\) and \(j\), with respective labels (i.e., returns) \(y_i\) and \(y_j\), and ranks \(\hat{\pi}_i\) and \(\hat{\pi}_j\), we define

\begin{equation}
P(\mathbf{y}\mid\mathbf{s},\hat{\pi}) = \prod_{y_i > y_j} P(y_i > y_j \mid s_i, s_j, \hat{\pi}_i, \hat{\pi}_j)
\end{equation}

where
\begin{equation}
  P(y_i > y_j \mid s_i, s_j, \hat{\pi}_i, \hat{\pi}_j)
  = \left[\frac{1}{1 + e^{-\sigma(s_i - s_j)}}\right]^{\omega_{ij}(\hat{\pi}_i,\hat{\pi}_j,y_i,y_j,s_i,s_j)}
  \label{eq:bt_ic}
\end{equation}
and the exponent \(\omega_{ij}(\cdot)\) is the scaling factor (\(\Delta\)) given in Table~\ref{tab:delta_comparison}.

For example, when \(\omega_{ij}(\cdot) = 1\),
\[
l(\mathbf{y}, \mathbf{s}) = \sum_{y_i > y_j} \log_2\!\left(1 + e^{-\sigma (s_i - s_j)}\right).
\]
This is the pairwise LambdaRank loss. When $\omega_{ij}(\cdot) = \Delta \mathrm{NDCG}$, as defined in Eq.~\eqref{eq:ndcg_define}, the resulting loss is the one used in LambdaRank with the NDCG objective.

\begin{definition}[LambdaRankIC Loss]

Under the LambdaLoss framework, letting $\omega_{ij}(\cdot)=\Delta\,\text{RankIC}$ gives us the loss optimized in our LambdaRankIC algorithm. Formally, let $\tilde{\mathbf{y}}$ represent the true ranks presorted by the true label values $\mathbf{y}$, as defined in Eq.~\eqref{eq:rho} and let $\tilde{y}_i$ and $\tilde{y}_j$ denote the true ranks of items \(i\) and \(j\), respectively, where rank \(1\) is the most relevant and larger rank values indicate lower relevance. Replacing $\hat{\pi}_i$, $\hat{\pi}_j$ with $\hat{r}_i$ and $\hat{r}_j$ for notation consistency with Eq.~\eqref{eq:rho}, as they represent predicted ranks based on model scores, we have
\begin{equation}
  \omega_{ij}(.) = \Delta~RankIC
  = \frac{12|\hat{r}_j - \hat{r}_i|\,|\tilde{y}_i - \tilde{y}_j|}{n(n^2 - 1)}
\end{equation}
Hence, 
\begin{equation}
l_C(\mathbf{y},\mathbf{s})= \sum_{y_i>y_j}\frac{12|\hat{r}_j - \hat{r}_i|\,|\tilde{y}_i - \tilde{y}_j|}{n(n^2 - 1)}\,\log_2\!\left(1+e^{-\sigma(s_i-s_j)}\right)
\label{eq:loss_rankic}
\end{equation}
Eq.~\eqref{eq:loss_rankic} is the loss function we are optimizing under the LambdaLoss framework. We call it the \textbf{LambdaRankIC Loss}.
\end{definition}

\subsubsection{\textbf{LambdaRankIC Loss} as an Upper Bound}
We now examine to what extent the LambdaRankIC loss is related to the true RankIC. As we shall see, the LambdaRankIC loss is an upper bound on $1-\rho$.

\begin{proposition}[RankIC loss]
\label{prop:rankic_loss}
Let the rank of item $i$ be written as $\hat{r}_i = 1 + \sum_{j\ne i}\mathbb{I}_{s_i < s_j}$, where $\mathbb{I}_{s_i < s_j}$ is the indicator function equal to 1 when $s_i<s_j$ and zero otherwise. Defining the RankIC loss as $1-\rho$, we have: 
\begin{equation}
    1-\rho 
  = \frac{12}{n(n^2-1)}\sum_{\tilde{y}_i<\tilde{y}_j}
    (\tilde{y}_j-\tilde{y}_i)\,\mathbb{I}_{s_i<s_j}
\end{equation}
\end{proposition}

The proof of Proposition~\ref{prop:rankic_loss} can be found in the appendix. 
Proposition~\ref{prop:rankic_loss} expresses $1-\rho$ as a sum over discordant 
pairs weighted by the true rank separation $(\tilde{y}_j - \tilde{y}_i)$. Hence swaps between items far apart in true relevance contribute more to the cost than 
swaps between similarly ranked items.
$\mathbb{I}_{s_i < s_j}$ can also be viewed as the zero--one loss, which is bounded above by the logistic loss \citep{wang2018lambdaloss}:
\begin{equation}
    \mathbb{I}_{s_i<s_j} \leq \log_2\!\left(1+e^{-\sigma(s_i-s_j)}\right)
\end{equation}
Hence, we have:
\begin{equation}
\label{eq:logistic_surrogate}
     1-\rho \leq \frac{12}{n(n^2-1)}\sum_{\tilde{y}_i<\tilde{y}_j}
    (\tilde{y}_j-\tilde{y}_i)\,\log_2\!\left(1+e^{-\sigma(s_i-s_j)}\right) = l_{\mathrm{sg}}  
\end{equation}
Eq.~\eqref{eq:logistic_surrogate} establishes that $l_{\mathrm{sg}}$ is the logistic surrogate of the RankIC loss $1-\rho$. 
Assuming no ties in scores, it directly follows that
\begin{align}\label{eq:full_surrogate_chain}
  1-\rho \leq l_{\mathrm{sg}}
  &= \frac{12}{n(n^2-1)}\sum_{\tilde{y}_i<\tilde{y}_j}
    \underbrace{1}_{|\hat{r}_j-\hat{r}_i| \geq 1}
    \cdot(\tilde{y}_j-\tilde{y}_i)\,
    \log_2\!\bigl(1+e^{-\sigma(s_i-s_j)}\bigr) \nonumber \\
  &\leq \frac{12}{n(n^2-1)}\sum_{\tilde{y}_i<\tilde{y}_j}
    |\hat{r}_j-\hat{r}_i|\,|\tilde{y}_j-\tilde{y}_i|\,
    \log_2\!\bigl(1+e^{-\sigma(s_i-s_j)}\bigr)
  = l_C,
\end{align}

given that $|\hat{r}_j-\hat{r}_i|\geq 1$ for all pairs when there are no ties in scores. Eq. \eqref{eq:full_surrogate_chain} establishes that the LambdaRankIC loss $l_C$ is an upper bound on $1 - \rho$. 

\section{Simulations}\label{sec:simulation}
\paragraph{Noiseless Convergence Check}
We first evaluate LambdaRankIC on simulated data with a known ground-truth signal. This allows us to study convergence behavior and sensitivity to distributional assumptions in a controlled setting. We generate a panel of $T = 120$ monthly cross-sections (i.e., groups), each containing $N = 500$ stocks. The target return for stock~$i$ in the next period~$t+1$ is given by
\begin{equation}\label{eq:dgp}
  y_{i,t+1} = f(\mathbf{x}_{i,t}) + \varepsilon_{i,t+1},
\end{equation}
where $f$ is a deterministic signal function, $\mathbf{x}_{i,t}$ is a 10-dimensional feature vector, and $\varepsilon_{i,t+1}$ is the noise term. Both the features and the noise term are drawn i.i.d.\ from $\mathcal{N}(0,1)$. 
Further, we let $f(\mathbf{x}_{i,j}) = \boldsymbol{\beta}^\top \mathbf{x}_{i,j}$,
        where $\boldsymbol{\beta} \in \mathbb{R}^{p}$ is a fixed unit-norm
        coefficient vector drawn once from $\mathcal{N}(\mathbf{0}, \mathbf{I})$.
The first 80~periods serve as the training set and the remaining 40~periods as
the test set. 

We first verify that our custom LambdaRankIC implementation can recover the ranking structure implied by the data-generating process (DGP) by conducting a noiseless convergence check. Specifically, we remove the noise term and set the learning target to $y_{i,t+1} = f(\mathbf{x}_{i,t})$. We then train the model for 1,000 boosting rounds and repeat the experiment across 10 random seeds.

\begin{figure}[H]
\caption{Convergence without Noise}
\label{fig:convergence_noiseless}
    \centering
    \includegraphics[width=\textwidth]{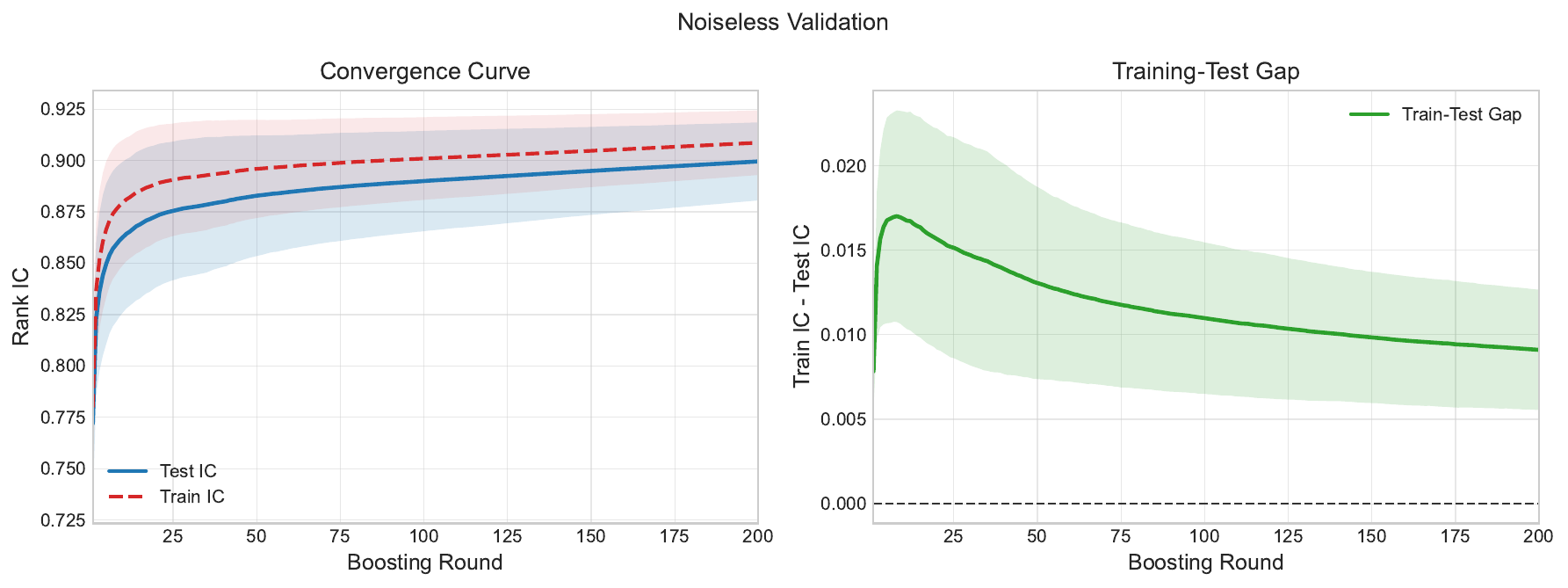}
    
\begin{minipage}{\textwidth}
    \vspace{6pt}
    \footnotesize \emph{Notes.} For readability, we plot only the first 200 boosting rounds. Learning rate $=0.01$, maximum depth $=6$.
  \end{minipage}
\end{figure}

Figure~\ref{fig:convergence_noiseless} reports the convergence curves in the noiseless setting. In the left panel, the model learns the noiseless ranking structure steadily and consistently. Averaged across seeds, the test Rank IC increases from $0.772$ at boosting round 1 to $0.949$ at round 1000, while the corresponding training Rank IC increases from $0.780$ to $0.953$ over the same period. Importantly, the train--test gap shown in the right panel remains small throughout training, rising modestly early on before narrowing steadily, from about $0.008$ at round 1 to about $0.004$ by round 1000. Overall, these dynamics suggest that the remaining gap to perfect ranking is unlikely to be driven by severe overfitting.
Taken together, these results provide a strong implementation sanity check. In the absence of noise, the objective steadily recovers the true ranking structure and achieves a mean best test Rank IC of $0.949$ (standard deviation $0.004$). Moreover, the best performance is attained essentially at the final boosting round across all seeds.

\paragraph{Convergence under Different Signal-to-Noise Regimes}
Next, we evaluate the model's convergence performance with more realistic noisy data across different signal-to-noise regimes. We increase the dimensionality of $\mathbf{x}_{i,t}$ to $p=100$ features. We calibrate the noise variance so that the signal-to-noise ratio
$\mathrm{SNR} \equiv \mathrm{Var}(f) / \mathrm{Var}(\varepsilon)$
takes one of three target values: High ($\mathrm{SNR} = 2.0$), Medium
($\mathrm{SNR} = 0.5$), or Low ($\mathrm{SNR} = 0.1$). The Low regime aims to simulate realistic equity return prediction, where monthly
out-of-sample $R^2$ values are typically below 1\%
\citep{gu2020empirical}.

We train the model with different learning rates ($\eta \in \{0.001, 0.01, 0.1\}$), boosted for 200 rounds, with all other hyperparameters held fixed.
Each configuration is repeated across 10 random seeds.
The evaluation metric is the cross-sectional Rank~IC (Spearman rank
correlation between predicted and realized returns), averaged across all
test-set periods.
Figure~\ref{fig:convergence_snr} reports convergence curves for the $p=10$ DGP (Panel~(a)) and the more challenging $p=100$ DGP (Panel~(b)). Across both panels, the learning dynamics are consistent: larger learning rates produce faster improvement and higher test Rank~IC within the fixed budget of 200 boosting rounds. For $p=10$, the pattern is clear under all SNR regimes: $\eta=0.1$ consistently achieves the highest end-of-training test Rank~IC, followed by $\eta=0.01$, while $\eta=0.001$ remains undertrained. Across SNR regimes, end-of-training test Rank~IC declines as the problem becomes harder (lower SNR). For example, in the high-SNR setting, the end Rank~IC is close to 0.9, whereas in the low-SNR setting it falls below 0.3.

Panel~(b) shows the same qualitative pattern but at a lower performance level, reflecting the increased difficulty of the $p=100$ setting. When $p=100$, all curves shift downward and the performance gap across learning rates widens substantially. In particular, the smaller learning rates ($\eta=0.001$ and $\eta=0.01$) remain far from the best-performing trajectory even at round 200, indicating undertraining under the fixed boosting budget. The largest learning rate, $\eta=0.1$, remains the strongest choice in all three SNR regimes, although its curve is still rising at the final round, suggesting that convergence is not yet fully saturated.

In Panel~(c), we report the train--test gap for the hardest setting ($p=100$, low SNR) under various learning rate settings ($\eta \in \{0.001, 0.01, 0.1\}$). Even in this regime, the model still learns meaningful ranking structure, as both training and test Rank~IC increase over boosting rounds across all three learning rates. For larger learning rates ($\eta=0.01$ and $\eta=0.1$), improvement in test Rank~IC eventually plateaus while the train--test gap widens, indicating progressive overfitting.
\begin{figure}[H]
\caption{Convergence Curves under Different Signal-to-Noise Ratio (SNR) Conditions}
  \label{fig:convergence_snr}
  \centering
  \begin{subfigure}[t]{\textwidth}
    \centering
    \includegraphics[width=\textwidth]{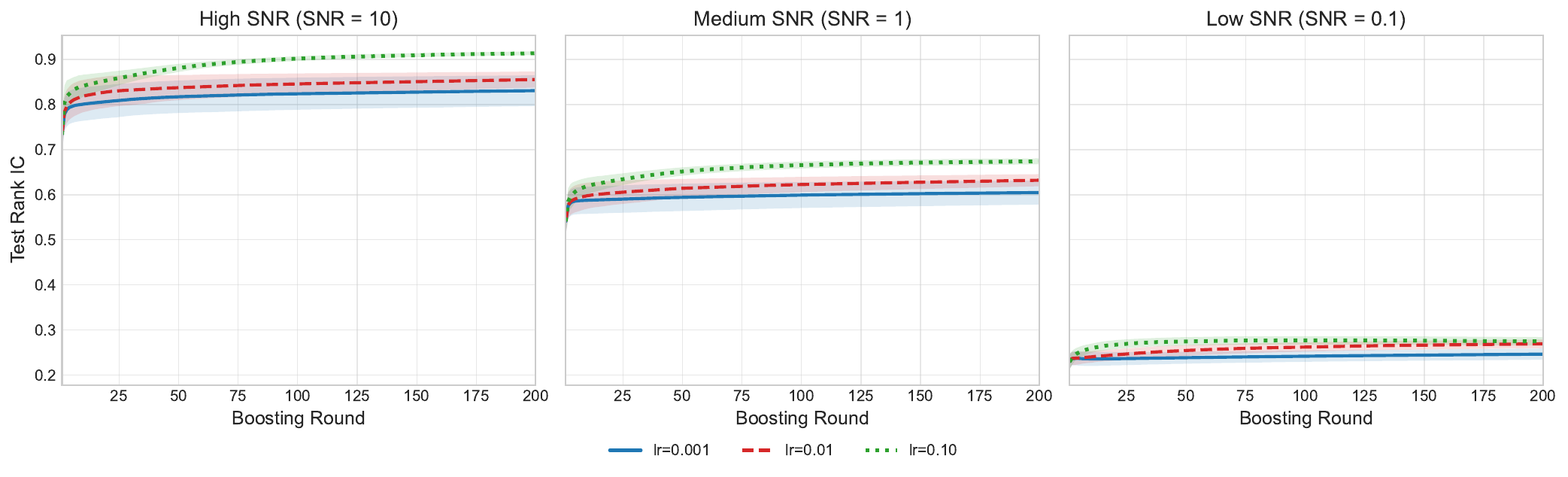}
    \caption{Convergence curves with 10 features}
  \end{subfigure}
  \vspace{6pt}
  \begin{subfigure}[t]{\textwidth}
    \centering
    \includegraphics[width=\textwidth]{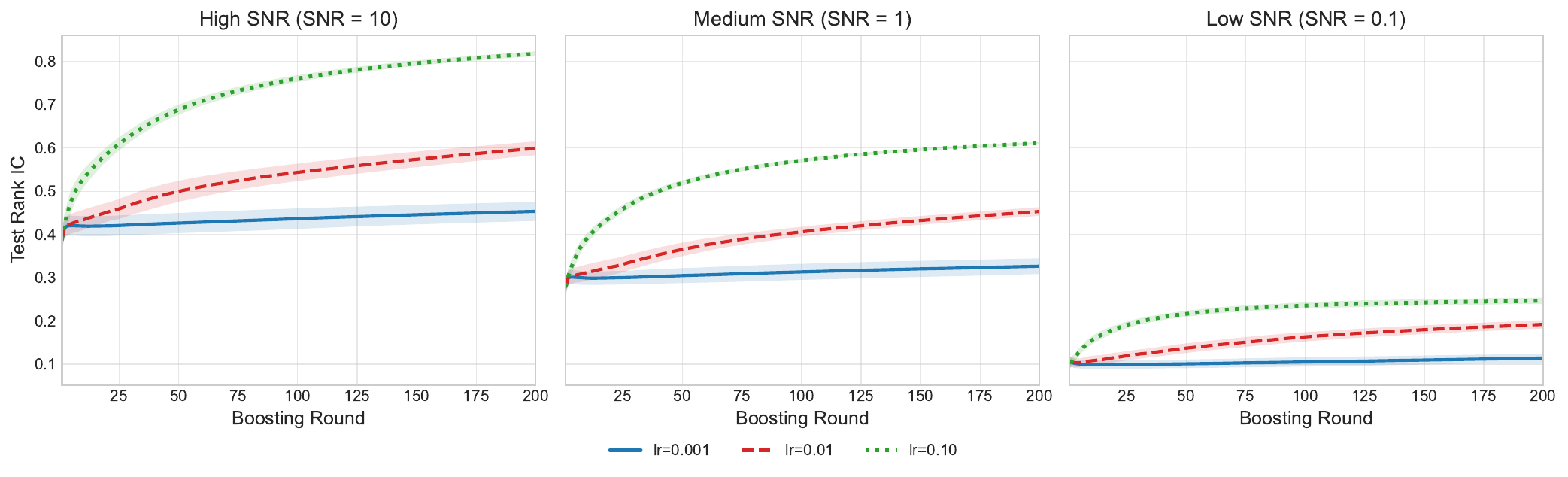}
    \caption{Convergence curves with 100 features}
  \end{subfigure}
  \vspace{6pt}
  \begin{subfigure}[t]{\textwidth}
    \centering
    \includegraphics[width=\textwidth]{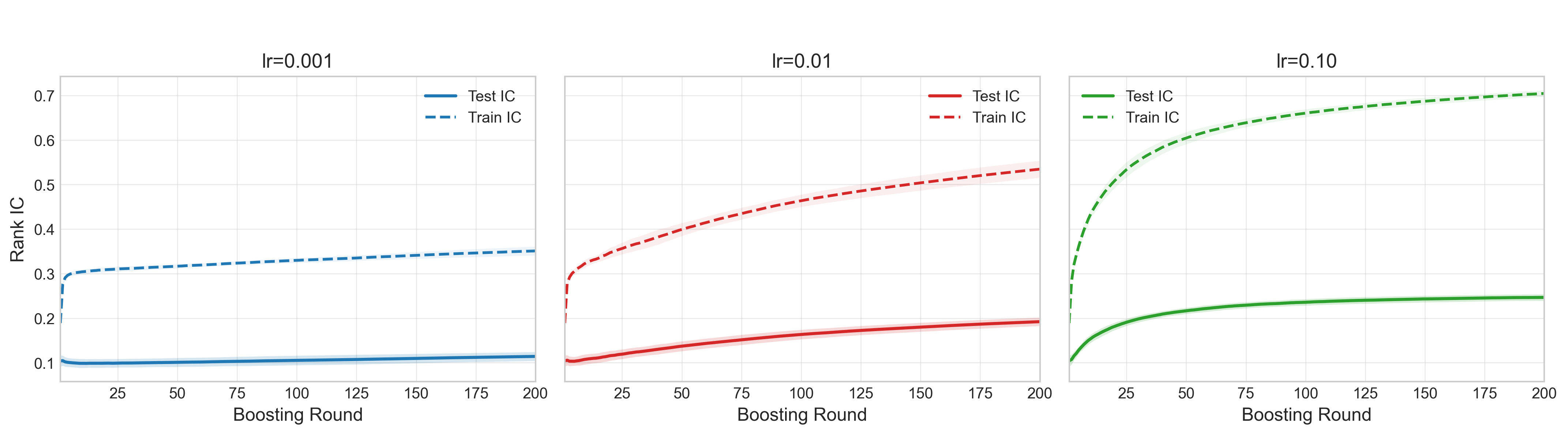}
    \caption{Train--test gap across different learning rate settings (Low SNR, 100 features)}
  \end{subfigure}
  \vspace{6pt}
  \begin{minipage}{\textwidth}
    \vspace{6pt}
    \footnotesize \emph{Notes.} (1) In all the plots, the y-axis is the test Rank~IC, while the x-axis is the boosting runs. (2) In Panel (a) and Panel (b), columns correspond to High, Medium, and Low SNR. (3) In Panel (C) columns correspond to different learning rate settings. (4)Shaded regions show $\pm 1$ standard deviation across 10 seeds.
  \end{minipage}
  
\end{figure}

\paragraph{Comparison with Existing Approaches in Financial Prediction}
In the previous two experiments, we assumed Gaussian noise. In practice, however, asset returns often exhibit heavy tails and excess kurtosis \citep{cont2001empirical,harvey2016man}. To better reflect this reality and investigate whether the proposed approach is more robust to heavy-tailed noise, our final simulation compares the proposed LambdaRankIC objective with two existing approaches under a heavy-tailed DGP: LambdaRank with NDCG objective and squared-error regression.
Specifically, we consider a linear DGP with $p=100$ features. We let $\varepsilon_{i,t+1}$ in Eq.~\eqref{eq:dgp} follow a Student-$t$ distribution with 5 degrees of freedom and rescale it to achieve the target SNR level.
We then train the same XGBoost model with different objectives: \texttt{rank:ic}, \texttt{rank:ndcg}, and \texttt{reg:squarederror}, while keeping the tree structure and boosting hyperparameters constant.

\begin{figure}[H]
\caption{Comparison with Existing Approaches}
\label{fig:comparison}
  \centering
  \begin{subfigure}{\textwidth}
    \centering
      
    \includegraphics[width=\textwidth]{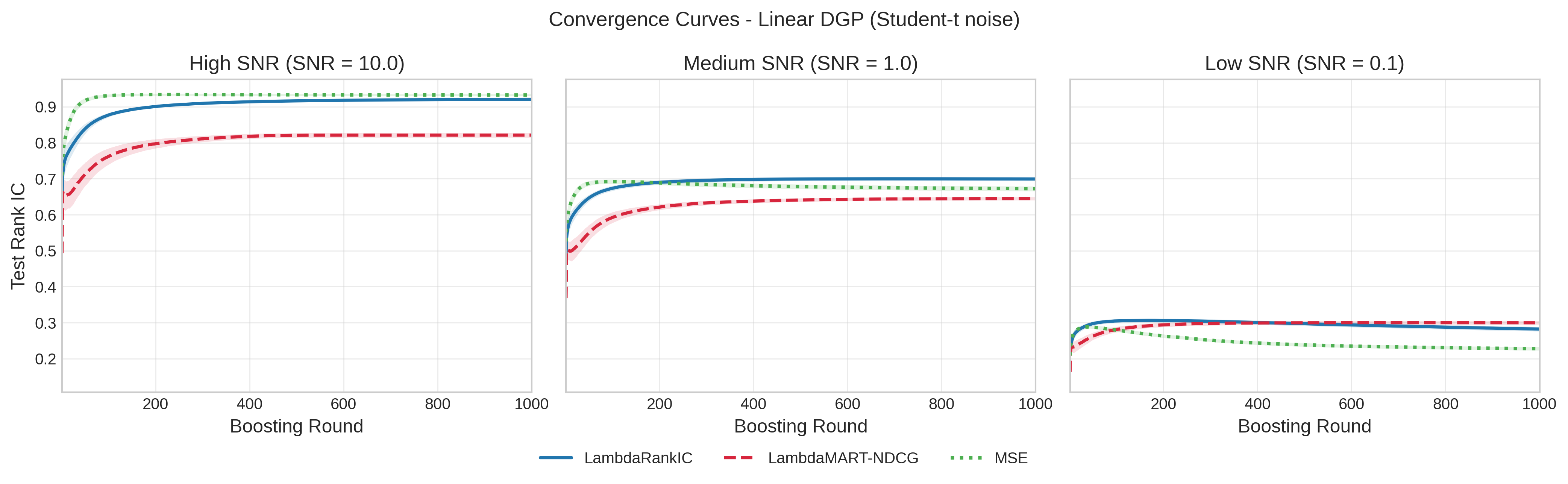}
    \caption{Convergence curves}
  \end{subfigure}

  \vspace{6pt}

  \begin{subfigure}{\textwidth}
    \centering
    
    \includegraphics[width=\textwidth]{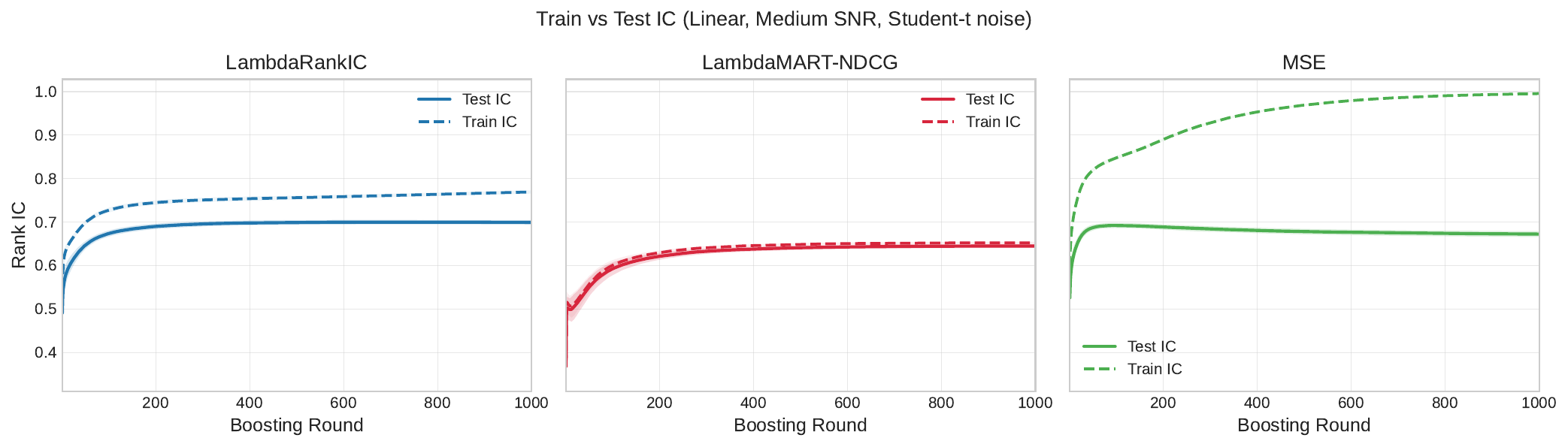}
    
    \caption{Train-test Gap: Medium SNR}
  \end{subfigure}
   \begin{minipage}{\textwidth}
    \vspace{6pt}
    \footnotesize \emph{Notes.} (1) Panel (a) shows the convergence curves across different signal-to-noise regimes. (2) Panel (b) shows train-test gap of the three learning objectives under the medium signal-to-noise regime.
  \end{minipage}

\end{figure}

Figure~\ref{fig:comparison} presents the learning curves and train--test gaps of the three models over 1000 boosting rounds, allowing us to examine both convergence dynamics and generalization behavior. Table~\ref{tab:linear_peak_round} reports the corresponding peak test Rank~IC values. The results support two related conclusions. First, the ranking-based objectives are more robust in the heavy-tailed, low-SNR regime, which is more representative of real-world financial data. As shown in Figure~\ref{fig:comparison}, the model with MSE objective peaks in early training rounds, and its train--test gap is substantially larger than those of the two ranking-based objectives in the low-SNR setting. Consistently, Table~\ref{tab:linear_peak_round} shows that MSE achieves a lower peak Rank~IC (0.2480) than LambdaRankIC (0.2803). This pattern suggests that MSE is more prone to overfitting and more sensitive to heavy-tailed noise. Second, comparing the two rank-based objectives, LambdaRankIC substantially outperforms LambdaRank-NDCG in test Rank~IC. This is consistent with the fact that the NDCG objective optimizes a top-heavy ranking criterion, whereas our target metric is full cross-sectional Rank~IC.

One might argue that the higher peak Rank~IC achieved by MSE and NDCG in high and medium SNR regimes suggests that the superiority of LambdaRankIC is confined solely to low-SNR environments. However, as shown in Table~\ref{tab:linear_peak_round}, LambdaRankIC reaches its highest Rank~IC at the 999th iteration, which is close to the predefined upper limit (1000) for boosting rounds. This fixed budget constraint suggests that the model has not yet fully converged, and that performance may improve further with additional boosting rounds. Furthermore, even under this constrained training setup, the gap between the highest overall peak Rank~IC and that of LambdaRankIC is marginal (less than 0.05), indicating no significant underperformance by our proposed method.

In summary, our simulations show that the proposed LambdaRankIC objective offers clear advantages over existing methods. Compared with the common practice of predicting return levels using MSE loss, LambdaRankIC is less prone to overfitting and is less sensitive to noise, especially in challenging low-SNR regimes. Our results also indicate that off-the-shelf ranking objectives are often top-heavy, emphasizing only the highest-ranked items. This objective mismatch leads to lower out-of-sample Rank~IC.

\begin{table}[H]
\centering
\caption{Peak Test Rank IC}
\label{tab:linear_peak_round}
\begin{tabular}{lcccccc}
\toprule
& \multicolumn{2}{c}{LambdaRankIC} & \multicolumn{2}{c}{NDCG} & \multicolumn{2}{c}{MSE} \\
\cmidrule(lr){2-3} \cmidrule(lr){4-5} \cmidrule(lr){6-7}
SNR & Peak IC & Peak Round & Peak IC & Peak Round & Peak IC & Peak Round \\
\midrule
High   & 0.8207 & 999   & 0.6697 & 568   & \textbf{0.8666} & 235 \\
Medium & 0.6344 & 665   & 0.5165 & 964.5 & \textbf{0.6467} & 98.5 \\
Low    & \textbf{0.2803} & 171   & 0.2381 & 691.5 & 0.2480 & 38 \\
\bottomrule
\end{tabular}
\end{table}


\section{Benchmarking with Real Financial Data}
\label{sec:experiments}
\subsection{Dataset and Experimental Protocol}
\label{sec:dataset_protocol}
We evaluate the proposed methods in real financial predictions. We set up the task by generating investment signals for the next period. The raw signal is calculated from the next-period excess return, defined as
\begin{equation}
    r_{i,t+1}=\frac{p_{i,t+1}-p_{i,t}}{p_{i,t}}-r^f_{t+1}
\end{equation}
where $p_{it}$ is the stock price and $r^f_{t+1}$ is the risk-free rate. 
For regression-based models, this is the prediction target. For learning-to-rank models, this is the label for ranking.

The dataset is constructed from \citet{green2017characteristics} and contains 94 monthly firm-level characteristics. Following common practice in the empirical asset pricing literature \citep{kelly2019ipca,gu2020empirical}, we transform each characteristic into its cross-sectional rank. 
After dropping observations with missing values, the data
contains 2,746,083 stock-month observations spanning 21,396 securities from
January 1964 to December 2024. All models are trained and evaluated under the same rolling-window protocol. In each window, we use 120 months for training, 60 months for validation, and 12 months for testing, then advance the window by 12 months. The validation period is used for hyperparameter tuning.

We compare our method against three regression-based benchmarks: ordinary least-squares (OLS) regression, XGBoost regression with MSE loss (\texttt{reg:squarederror}), and neural-network regression.
OLS and neural-network models are implemented with scikit-learn.
For the neural-network baseline, we use a three hidden layer architecture following \citet{gu2020empirical}, with hidden layer sizes of 32, 16, and 8. We also compare our LambdaRankIC approach against two off-the-shelf learning-to-rank methods using XGBoost's LambdaRank implementation: standard pairwise ranking (\texttt{rank:pairwise}) and NDCG-optimized ranking (\texttt{rank:ndcg}).
The regression baselines are trained on raw future excess returns. The
learning-to-rank models use the same inputs and training, validation, and test splits, but we treat each month as one group for ranking (i.e., one query).

We report four ranking metrics: average monthly Rank~IC ($\overline{IC}$), its standard
deviation $\mathrm{Std}(IC)$, $\mathrm{ICIR} = \overline{IC}\,/\,\mathrm{Std}(IC)$, and $\mathrm{NDCG@100}$.
We also report portfolio-based performance. Specifically, following the existing literature \citep{gu2020empirical}, each month we sort assets into deciles based on model scores and construct a value-weighted long--short portfolio that buys the top decile and sells the bottom decile, rebalanced monthly.
We report the portfolio's average monthly return ($\mathrm{Return}$),
volatility ($\mathrm{Vol.}$), annualized Sharpe ratio ($\mathrm{Sharpe}$), and maximum drawdown
($\mathrm{MDD}$). All metrics except the Sharpe ratio are reported as percentages.

\subsection{Results}
\label{sec:results}
Table~\ref{tab:joc-benchmark} reports the consolidated out-of-sample benchmark
results. The results support the conclusion that directly optimizing an IC-aligned objective yields the best cross-sectional ranking performance. Our approach (RankIC) is the strongest model on three of the four ranking metrics:
it attains the highest average IC, the highest ICIR, and the best NDCG@100. Compared with the regression-based methods, learning-to-rank methods overall achieve higher average IC, ICIR, and NDCG@100. Both pairwise and NDCG-based methods achieve an average Rank~IC of about 0.08, while our method attains a substantially higher IC of 0.115, an increase of over 30\% relative to these baselines.
As expected, regression-based methods achieve considerably lower IC values, averaging about 0.044.
Turning to portfolio returns, our method achieves the highest monthly average return at 2.22\%. Its annualized Sharpe ratio is also the highest at approximately 0.92. While the volatility is moderately higher compared with regression-based methods, the return is substantially higher, leading to a superior overall Sharpe ratio.

\begin{table}[t]
\centering
\scriptsize
\setlength{\tabcolsep}{4pt}
\caption{Out-of-sample Benchmark Results}
\label{tab:joc-benchmark}
\begin{tabular}{lcccccccc}
\toprule
& \multicolumn{4}{c}{Ranking Metrics} & \multicolumn{4}{c}{Portfolio Returns} \\
\cmidrule(lr){2-5} \cmidrule(lr){6-9}
Method &  $\overline{IC}\uparrow$ & Std(IC) $\downarrow$ & ICIR $\uparrow$ & NDCG@100 $\uparrow$ & Return (\%) $\uparrow$ & Vol. (\%) $\downarrow$ & Sharpe $\uparrow$ & MDD (\%) $\downarrow$ \\
\midrule
\multicolumn{9}{l}{\textit{Regression-based}} \\
OLS Regression  & 0.0471 & 0.0706 & 0.6663 & 0.5077 & 0.97 & 4.56 & 0.740 & \textbf{43.78} \\
XGBoost Regression & 0.0418 & 0.0917 & 0.4561 & 0.4924 & 1.25 & 6.21 & 0.696 & 58.29 \\
MLP Regression     & 0.0428 & \textbf{0.0531} & 0.8059 & 0.4967 & 1.08 & \textbf{4.52} & 0.831 & 46.28 \\
\addlinespace
\multicolumn{9}{l}{\textit{Learning-to-rank-based}} \\
Pairwise         & 0.0828 & 0.1153 & 0.7181 & 0.5328 & 1.42 & 8.68 & 0.566 & 81.59 \\
NDCG             & 0.0863 & 0.1063 & 0.8122 & 0.5312 & 1.20 & 8.32 & 0.501 & 83.24 \\
RankIC (ours)    & \textbf{0.1148} & 0.1114 & \textbf{1.0308} & \textbf{0.5415} & \textbf{2.22} & 8.31 & \textbf{0.923} & 64.29 \\
\bottomrule
\end{tabular}
\parbox{0.85\linewidth}{\footnotesize \vspace{6pt} 
\linespread{0.9}\selectfont \emph{Notes.} (1) $\uparrow$ ($\downarrow$) indicates higher (lower) is better; \textbf{bold} marks the best value in each column. (2) Ranking metrics are computed from next-month model predictions. (3) Portfolio metrics are based on value-weighted top-minus-bottom decile returns formed monthly using model scores. (4) Return, volatility, and maximum drawdown are reported in monthly percent. (5) Sharpe ratios are annualized: $\text{Sharpe} = Return / Vol. \times \sqrt{12}$.}

\end{table}

To gain deeper insights into the model predictions and their corresponding portfolio performance, we stratify the cross-sectional stocks into ten deciles for each method, as reported in Table~\ref{tab:deciles}. Stocks with the highest predicted scores are allocated to Decile 10, while those with the lowest are placed in Decile 1. Several key observations emerge from this analysis.
First, across all methods, decile returns and Sharpe ratios generally increase from Decile 1 to Decile 10, while volatilities broadly decrease. This trend validates the cross-sectional sorting capability of all evaluated models. Second, as shown in Panel A, a comparison among the learning-to-rank methods reveals that our proposed RankIC more effectively identifies underperforming stocks than the other two learning-to-rank methods, producing strongly negative returns in the lowest decile (Decile 1). This stronger short signal significantly enhances profitability when constructing long--short portfolios when shorting the bottom decile.
Third, compared with traditional regression models, LTR methods successfully generate profitable short signals. The regression methods fail in this regard: the lowest deciles of all three regression models yield positive returns, which severely degrades their long--short portfolio performance. Finally, although RankIC incurs a higher maximum drawdown than the regression baselines, it maintains the lowest drawdown among all LTR models and achieves the highest Sharpe ratio across all six evaluated methods.

\begin{table}[h]
\scriptsize
\centering
\caption{Performance Metrics by Decile}
\label{tab:deciles}
\begin{tabular*}{\linewidth}{@{\extracolsep{\fill}}lcccccccccccc}
\toprule
\multicolumn{13}{l}{\textbf{Panel A: LTR Methods}} \\
\midrule
 & \multicolumn{4}{c}{LTR: RankIC} & \multicolumn{4}{c}{LTR: Pairwise} & \multicolumn{4}{c}{LTR: NDCG} \\
\cmidrule(lr){2-5} \cmidrule(lr){6-9} \cmidrule(lr){10-13}
Decile & Ret & Vol & SR & MDD & Ret & Vol & SR & MDD & Ret & Vol & SR & MDD \\
\midrule
1 & -1.20 & 10.08 & -0.41 & 100.00 & -0.34 & 10.27 & -0.12 & 99.70 & -0.27 & 9.95 & -0.09 & 99.44 \\
2 & -0.10 & 8.53 & -0.04 & 97.97 & 0.03 & 8.54 & 0.01 & 97.27 & 0.38 & 8.05 & 0.16 & 90.21 \\
3 & 0.44 & 7.35 & 0.21 & 91.56 & 0.31 & 6.76 & 0.16 & 91.99 & 0.54 & 7.30 & 0.26 & 85.47 \\
4 & 0.42 & 5.81 & 0.25 & 74.57 & 0.61 & 6.05 & 0.35 & 83.79 & 0.52 & 6.32 & 0.29 & 84.97 \\
5 & 0.59 & 5.36 & 0.38 & 71.95 & 0.75 & 5.28 & 0.49 & 69.69 & 0.69 & 5.42 & 0.44 & 71.69 \\
6 & 0.69 & 4.82 & 0.49 & 67.62 & 0.79 & 4.84 & 0.57 & 63.98 & 0.82 & 4.73 & 0.60 & 55.84 \\
7 & 0.86 & 4.52 & 0.66 & 46.44 & 0.93 & 4.62 & 0.70 & 44.71 & 0.81 & 4.79 & 0.59 & 47.91 \\
8 & 0.88 & 4.34 & 0.71 & 41.70 & 0.85 & 4.69 & 0.63 & 45.21 & 0.80 & 4.61 & 0.60 & 50.39 \\
9 & 0.94 & 4.37 & 0.75 & 48.52 & 0.89 & 4.84 & 0.64 & 49.26 & 1.02 & 4.65 & 0.76 & 45.23 \\
10 & 1.01 & 4.38 & 0.80 & 50.26 & 1.07 & 5.07 & 0.73 & 55.82 & 0.93 & 4.86 & 0.67 & 55.52 \\
\midrule
H--L & 2.22 & 8.32 & 0.92 & 64.29 & 1.42 & 8.69 & 0.57 & 81.59 & 1.20 & 8.32 & 0.50 & 83.24 \\
\midrule
\multicolumn{13}{l}{\textbf{Panel B: Regression Methods}} \\
\midrule
 & \multicolumn{4}{c}{Regression: OLS} & \multicolumn{4}{c}{Regression: XGBoost} & \multicolumn{4}{c}{Regression: MLP} \\
\cmidrule(lr){2-5} \cmidrule(lr){6-9} \cmidrule(lr){10-13}
Decile & Ret & Vol & SR & MDD & Ret & Vol & SR & MDD & Ret & Vol & SR & MDD \\
\midrule
1 & 0.19 & 5.66 & 0.12 & 85.70 & 0.12 & 7.90 & 0.05 & 97.03 & 0.16 & 6.71 & 0.08 & 96.75 \\
2 & 0.47 & 4.86 & 0.34 & 67.56 & 0.52 & 6.69 & 0.27 & 87.93 & 0.49 & 5.05 & 0.34 & 77.99 \\
3 & 0.75 & 4.59 & 0.57 & 52.54 & 0.59 & 5.39 & 0.38 & 66.29 & 0.67 & 4.59 & 0.50 & 51.62 \\
4 & 0.68 & 4.52 & 0.52 & 53.24 & 0.48 & 5.30 & 0.31 & 65.98 & 0.70 & 4.47 & 0.54 & 54.29 \\
5 & 0.78 & 4.66 & 0.58 & 52.64 & 0.84 & 5.24 & 0.56 & 52.26 & 0.81 & 4.38 & 0.64 & 44.92 \\
6 & 0.85 & 4.65 & 0.64 & 42.74 & 0.82 & 5.33 & 0.53 & 61.22 & 0.76 & 4.68 & 0.56 & 53.33 \\
7 & 0.93 & 4.81 & 0.67 & 53.54 & 0.84 & 4.86 & 0.60 & 54.07 & 0.88 & 4.88 & 0.63 & 51.08 \\
8 & 0.87 & 5.17 & 0.58 & 53.44 & 0.83 & 5.19 & 0.55 & 59.37 & 0.71 & 4.79 & 0.52 & 59.83 \\
9 & 1.03 & 5.38 & 0.66 & 61.73 & 0.95 & 5.40 & 0.61 & 53.63 & 1.00 & 5.32 & 0.65 & 54.02 \\
10 & 1.17 & 5.67 & 0.71 & 56.85 & 1.37 & 5.89 & 0.81 & 47.31 & 1.25 & 6.39 & 0.68 & 48.56 \\
\midrule
H--L & 0.97 & 4.56 & 0.74 & 43.78 & 1.25 & 6.21 & 0.70 & 58.29 & 1.08 & 4.53 & 0.83 & 46.28 \\
\bottomrule
\end{tabular*}

\vspace{3pt}
\parbox{\linewidth}{\footnotesize \linespread{0.9}\selectfont \emph{Notes.} 
This table reports decile portfolio performance based on model predictions. We report monthly return (Ret), monthly volatility (Vol), annualized Sharpe ratio (SR), and maximum drawdown (MDD) for each method. The high-minus-low (H--L) portfolio is constructed by taking a long position in stocks in the highest decile (Decile 10) and a short position in stocks in the lowest decile (Decile 1). All metrics except the Sharpe ratio are reported in percent.
}
\end{table}

We further visualize the cumulative returns of the long--short portfolios over the sample period in Figure~\ref{fig:cumulative_returns}. Notably, the portfolio constructed using RankIC begins to consistently outperform the other portfolios around 1985, and the performance gap between RankIC and the benchmark methods widens steadily thereafter. This sustained outperformance in cumulative returns provides further evidence of the effectiveness of our proposed method.

In contrast, the NDCG-based method underperforms all other approaches, particularly after 2008, highlighting the potential limitations of its top-heavy objective. The Pairwise method performs robustly, maintaining the second-highest cumulative return trajectory, especially in the later period. Among the regression-based approaches, the OLS (Normal) Regression model delivers the lowest returns, while the MLP and XGBoost models converge to similar performance levels by the end of the sample period.
Furthermore, the RankIC portfolio has remained notably resilient during recessionary periods, such as the Great Recession of 2007–2009 and the outbreak of the 2020 pandemic. This strong robustness suggests that our proposed method is well suited to navigating severe market downturns.
\begin{figure}[t]
\caption{Cumulative Returns}
\label{fig:cumulative_returns}
    \centering
    \includegraphics[width=\textwidth]{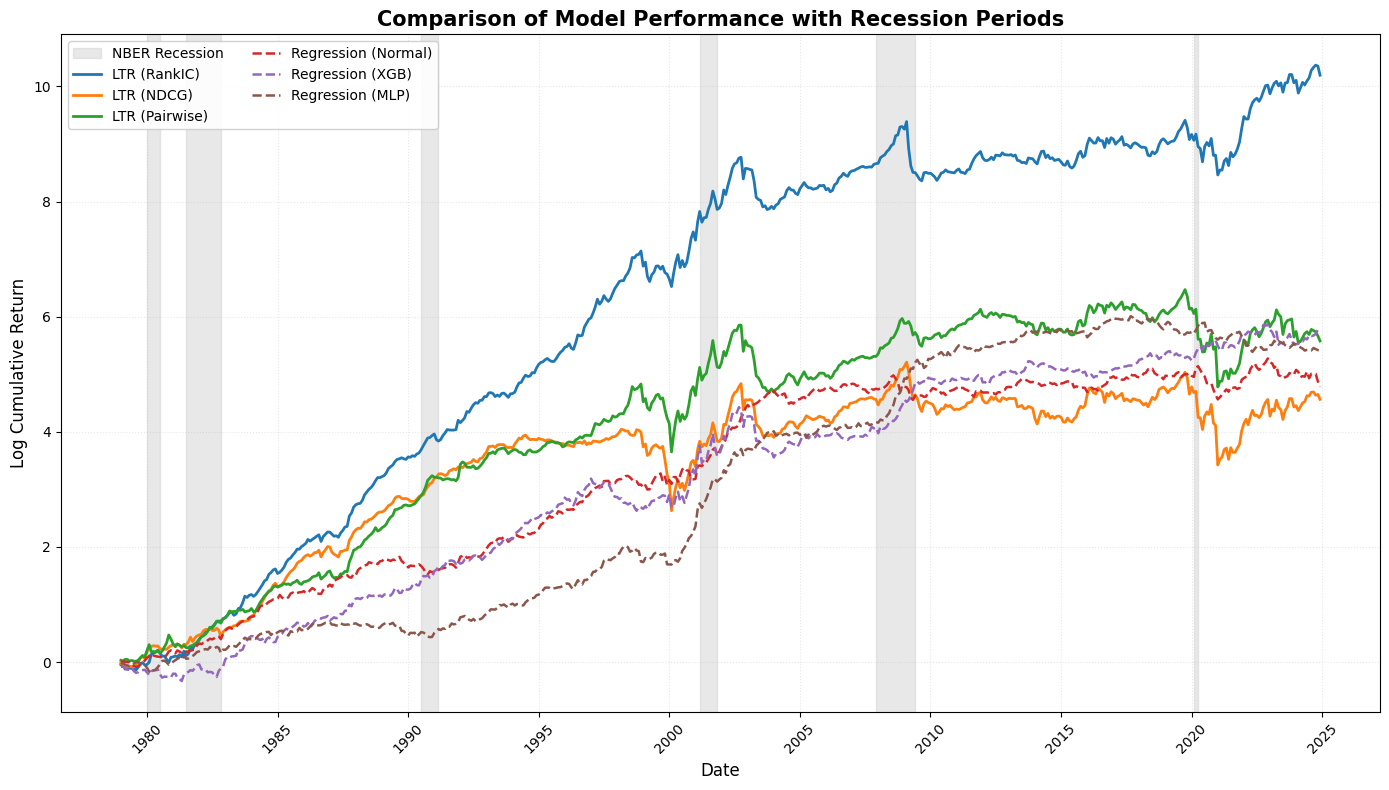}
    \begin{minipage}{\textwidth}
    \vspace{4pt}
    \rule{\textwidth}{0.4pt}
    \vspace{2pt}
    \footnotesize Note: The shaded periods denote recessions as defined by the National Bureau of Economic Research (NBER).    
  \end{minipage}
  
\end{figure}
\section{Conclusion}
\label{sec:conclusion}

Motivated by the need to directly optimize Rank IC in financial applications, we propose LambdaRankIC, a novel learning-to-rank objective that directly targets the Spearman rank correlation between model predictions and realized asset returns. We derive a closed-form expression for the lambda gradients within the LambdaRank framework and analytically show that the proposed objective optimizes a coarse upper bound of Rank IC.
We conduct extensive experiments using both simulated data and real financial market data. In simulations with noiseless data, we demonstrate that our approach is able to recover the underlying ranked structure of the data-generating process. Across all simulation settings, LambdaRankIC consistently outperforms existing approaches, including standard regression-based methods and learning-to-rank models optimized using the NDCG objective.
Furthermore, using a large-scale empirical dataset comprising 94 firm characteristics over a 30-year period, we show that LambdaRankIC achieves the highest out-of-sample Rank IC and ICIR, as well as superior portfolio performance in terms of returns and Sharpe ratios, relative to both regression models with various specifications and existing learning-to-rank methods.
Overall, our results demonstrate that LambdaRankIC represents a meaningful extension of existing learning-to-rank methodologies and provides a powerful new tool for ranking-based prediction problems in finance.

\subsection{Limitations and Future Directions}

Our approach is not without limitations. First, from a methodological perspective, LambdaRankIC inherits the same limitations from the original LambdaRank framework. In particular, the derived $\lambda$ values are pseudo-gradients rather than true gradients, since the Rank~IC objective is a non-smooth function of the model scores. Nevertheless, \citet{donmez2009efficiently} demonstrate that LambdaMART converges to a fixed point of the gradient under mild conditions, although no global convergence guarantees are available. We explicitly acknowledge this limitation and address it empirically by presenting convergence curves in Section~\ref{sec:experiments}. Moreover, we partially mitigate this concern by showing that the LambdaRankIC algorithm optimizes a coarse upper bound of $1 - \text{Rank IC}$, such that minimizing the LambdaRankIC loss indirectly maximizes Rank~IC. Relatedly, tighter upper bounds of Rank~IC may exist. For example, \citet{wang2018lambdaloss} derive alternative weighting schemes for the $\Delta$ terms in LambdaRank to more closely approximate target ranking metrics. However, they also caution that employing tighter surrogate bounds may increase the risk of overfitting. Exploring this trade-off in the context of Rank~IC optimization remains an interesting direction for future research.

Second, one of the key motivations of our work is the non-differentiability of the ranking operator. A promising direction is therefore to incorporate differentiable sorting and ranking operators \citep{blondel2020fast} into gradient boosting, enabling theoretically grounded gradients for Spearman rank correlation.

Third, our experiments focus on cross-sectional ranking, where each date defines a group of assets to be ranked. The framework extends naturally to time-series ranking. For example, we can treat an asset as a query and rank its expected returns across future periods. This formulation aligns with market-timing applications and time-series IC evaluation, and would allow the same RankIC-optimized objective to be applied to a structurally different prediction task.

Finally, our current implementation is limited to CPU-based training. Developing a parallelized or CUDA-enabled implementation to support GPU-accelerated training on large-scale datasets is a natural next step.

\clearpage
\bibliographystyle{abbrvnat}
\bibliography{refs}

\appendix




\section{Proof of Proposition \ref{prop:rankic_loss}}

\begin{proof}
By the definition of $\rho$ in Eq.~(9) and the fact that
\[
\bar{\hat r}=\bar{\tilde y}=\frac{n+1}{2},
\]
we have
\[
\rho
=
\frac{12}{n(n^2-1)}
\left(
\sum_{i=1}^n \hat r_i \tilde y_i
-
n\frac{(n+1)^2}{4}
\right).
\]

Using
\[
\hat r_i = 1+\sum_{j\neq i}\mathbb{I}_{\,s_i<s_j},
\]
we obtain
\[
\sum_{i=1}^n \hat r_i\tilde y_i
=
\sum_{i=1}^n \tilde y_i
+
\sum_{i=1}^n \sum_{j\neq i}\tilde y_i\,\mathbb{I}_{\,s_i<s_j}.
\]
Now consider the double sum
\[
\sum_{i=1}^n \sum_{j\neq i}\tilde y_i\,\mathbb{I}_{\,s_i<s_j}.
\]
Each ordered pair $(i,j)$ with $i\neq j$ appears exactly once in this sum. 
To rewrite the sum in terms of unordered pairs, group together the two ordered pairs
$(i,j)$ and $(j,i)$ associated with the same unordered pair $\{i,j\}$. 
For each unordered pair, relabel the indices so that
\[
\tilde y_i<\tilde y_j.
\]
Then the contribution of the ordered pair $(i,j)$ is
\[
\tilde y_i\,\mathbb{I}_{\,s_i<s_j},
\]
while the contribution of the ordered pair $(j,i)$ is
\[
\tilde y_j\,\mathbb{I}_{\,s_j<s_i}.
\]
Therefore, for each unordered pair $\{i,j\}$ such that $\tilde y_i<\tilde y_j$, the total contribution is
\[
\tilde y_i\,\mathbb{I}_{\,s_i<s_j}
+
\tilde y_j\,\mathbb{I}_{\,s_j<s_i}.
\]

Assuming that there are no ties in the scores, we have
\[
\mathbb{I}_{\,s_j<s_i}=1-\mathbb{I}_{\,s_i<s_j},
\]
and hence
\begin{align*}
\tilde y_i\,\mathbb{I}_{\,s_i<s_j}
+
\tilde y_j\,\mathbb{I}_{\,s_j<s_i}
&=
\tilde y_i\,\mathbb{I}_{\,s_i<s_j}
+
\tilde y_j\bigl(1-\mathbb{I}_{\,s_i<s_j}\bigr) \\
&=
\tilde y_j-(\tilde y_j-\tilde y_i)\mathbb{I}_{\,s_i<s_j}.
\end{align*}

Therefore,
\[
\sum_{i=1}^n \sum_{j\neq i}\tilde y_i\,\mathbb{I}_{\,s_i<s_j}
=
\sum_{\tilde y_i<\tilde y_j}
\left[
\tilde y_j-(\tilde y_j-\tilde y_i)\mathbb{I}_{\,s_i<s_j}
\right],
\]
where $\sum_{\tilde y_i<\tilde y_j}$ denotes summation over all unordered pairs $\{i,j\}$, indexed so that $\tilde y_i<\tilde y_j$.

Substituting this into the previous display gives
\begin{align*}
\sum_{i=1}^n \hat r_i\tilde y_i
&=
\sum_{i=1}^n \tilde y_i
+
\sum_{\tilde y_i<\tilde y_j}
\left[
\tilde y_j-(\tilde y_j-\tilde y_i)\mathbb{I}_{\,s_i<s_j}
\right] \\
&=
\sum_{i=1}^n \tilde y_i
+
\sum_{\tilde y_i<\tilde y_j}\tilde y_j
-
\sum_{\tilde y_i<\tilde y_j}
(\tilde y_j-\tilde y_i)\mathbb{I}_{\,s_i<s_j}.
\end{align*}

Since $\tilde y$ is a permutation of $\{1,\ldots,n\}$,
\[
\sum_{i=1}^n \tilde y_i
=
\sum_{k=1}^n k
=
\frac{n(n+1)}{2}.
\]

Moreover, to evaluate
\[
\sum_{\tilde y_i<\tilde y_j}\tilde y_j,
\]
fix a value $k\in\{1,\ldots,n\}$ and consider the unique index $j$ such that
\[
\tilde y_j = k.
\]
Here, $\tilde y_j$ is the true rank of sample $j$. In the sum above, the term $\tilde y_j$ is counted once for every index $i$ satisfying
\[
\tilde y_i < \tilde y_j = k.
\]
Since $\tilde y$ is a permutation of $\{1,\ldots,n\}$, the ranks smaller than $k$ are exactly
\[
1,2,\ldots,k-1,
\]
and each of these ranks appears exactly once. Therefore, there are exactly $k-1$ indices $i$ such that $\tilde y_i<k$. Hence, the value $k=\tilde y_j$ is added exactly $k-1$ times in the sum, so its total contribution is
\[
(k-1)k.
\]
Summing over all possible rank values $k=1,\ldots,n$, we obtain
\[
\sum_{\tilde y_i<\tilde y_j}\tilde y_j
=
\sum_{k=1}^n (k-1)k.
\]

Therefore,
\[
\sum_{\tilde y_i<\tilde y_j}\tilde y_j
=
\sum_{k=1}^n (k-1)k
=
\frac{n(n+1)(2n+1)}{6}-\frac{n(n+1)}{2}.
\]

Hence,
\[
\sum_{i=1}^n \hat r_i\tilde y_i
=
\frac{n(n+1)(2n+1)}{6}
-
\sum_{\tilde y_i<\tilde y_j}
(\tilde y_j-\tilde y_i)\mathbb{I}_{\,s_i<s_j}.
\]

Substituting this into the expression for $\rho$ yields
\[
\rho
=
\frac{12}{n(n^2-1)}
\left[
\frac{n(n+1)(2n+1)}{6}
-
n\frac{(n+1)^2}{4}
-
\sum_{\tilde y_i<\tilde y_j}
(\tilde y_j-\tilde y_i)\mathbb{I}_{\,s_i<s_j}
\right].
\]
Since
\[
\frac{n(n+1)(2n+1)}{6}
-
n\frac{(n+1)^2}{4}
=
\frac{n(n^2-1)}{12},
\]
it follows that
\[
\rho
=
1-
\frac{12}{n(n^2-1)}
\sum_{\tilde y_i<\tilde y_j}
(\tilde y_j-\tilde y_i)\mathbb{I}_{\,s_i<s_j},
\]
and therefore
\[
1-\rho
=
\frac{12}{n(n^2-1)}
\sum_{\tilde y_i<\tilde y_j}
(\tilde y_j-\tilde y_i)\mathbb{I}_{\,s_i<s_j}.
\]
\end{proof}

\section{Hyperparameters Used in Simulation}

\label{ec:hyper}
\begin{table}[ht]
\centering
\label{tab:hyperparameters}
\begin{tabular}{ll}
\toprule
Hyperparameter & Value \\
\midrule
max\_depth & 8 \\
learning\_rate & \{0.001, 0.01, 0.1\} \\
subsample & 1.0 \\
colsample\_bytree & 1.0 \\
min\_child\_weight & 1 \\
tree\_method & hist \\
\bottomrule
\end{tabular}
\end{table}

\section{Characteristics used in Return Prediction}
\begin{longtable}{@{}p{0.22\textwidth}p{0.72\textwidth}@{}}
\toprule
Variable & Description \\
\midrule
\endfirsthead

\toprule
Variable & Description \\
\midrule
\endhead

\midrule
\multicolumn{2}{r}{\textit{Continued on next page}} \\
\endfoot

\bottomrule
\endlastfoot

absacc & Absolute value of \textit{acc} \\
acc & Working capital accruals \\
aeavol & Abnormal earnings announcement volume \\
age & \# of years since first Compustat coverage \\
agr & Asset growth rate \\
baspread & Bid-ask spread \\
beta & Beta \\
betasq & Market beta squared \\
bm & Book-to-market ratio \\
bm\_ia & Industry-adjusted book-to-market ratio \\
cash & Cash holdings \\
cashdebt & Cash flow to debt \\
cashpr & Cash productivity \\
cfp & Cash flow to price ratio \\
cfp\_ia & Industry-adjusted cash flow to price ratio \\
chatoia & Industry-adjusted change in asset turnover \\
chcsho & \% Change in shares outstanding \\
chempia & Industry-adjusted change in \# of employees \\
chinv & Change in inventory \\
chmom & Change in 6-month momentum \\
chpmia & Industry-adjusted change in profit margin \\
chtx & \% Change in tax expense \\
cinvest & Corporate investment \\
convind & Convertible debt indicator \\
currat & Current ratio \\
depr & Depreciation / PP\&E \\
divi & Dividend initiation \\
divo & Dividend omission \\
dolvol & Dollar trading volume \\
dy & Dividend to price \\
ear & Earnings announcement return \\
egr & Growth in common shareholder equity \\
ep & Earnings to price \\
gma & Gross profitability \\
grCAPX & \% change in capital expenditures \\
grltnoa & Growth in long term net operating assets \\
herf & Industry sales concentration \\
hire & Employee growth rate \\
idiovol & Idiosyncratic return volatility \\
ill & Illiquidity \\
indmom & Industry momentum \\
invest & Capital expenditures and inventory \\
lev & Leverage \\
lgr & Change in total liabilities \\
maxret & Maximum daily return \\
mom12m & 12-month momentum \\
mom1m & 1-month momentum \\
mom36m & 36-month momentum \\
mom6m & 6-month momentum \\
ms & Financial statement score \\
mve & Size \\
mve\_ia & Industry-adjusted size \\
nincr & Number of earnings increases \\
operprof & Operating profitability \\
orgcap & Organizational capital \\
pchcapx\_ia & Industry adjusted \% change in capital expenses \\
pchcurrat & \% change in current ratio \\
pchdepr & \% change in depreciation \\
pchgm\_pchsale & \% change in gross margin - \% change in sales \\
pchquick & \% change in quick ratio \\
pchsale\_pchinvt & \% change in sales - \% change in inventory \\
pchsale\_pchrect & \% change in sales - \% change in A/R \\
pchsale\_pchxsga & \% change in sales - \% change in SG\&A \\
pchsaleinv & \% change in sales-to-inventory \\
pctacc & Percent accruals \\
pricedelay & Price delay \\
ps & Financial statements score \\
quick & Quick ratio \\
rd & R\&D increase \\
rd\_mve & R\&D to market capitalization \\
rd\_sale & R\&D to sales \\
realestate & Real estate holdings \\
retvol & Return volatility \\
roaq & Return on assets \\
roavol & Earnings volatility \\
roeq & Return on equity \\
roic & Return on invested capital \\
rsup & Revenue surprise \\
salecash & Sales to cash \\
saleinv & Sales to inventory \\
salerec & Sales to receivables \\
secured & Secured debt \\
securedind & Secured debt indicator \\
sgr & Sales growth \\
sin & Sin stocks \\
sp & Sales to price \\
std\_dolvol & Volatility of liquidity (dollar trading volume) \\
std\_turn & Volatility of liquidity (share turnover) \\
stdacc & Accrual volatility \\
stdcf & Cash flow volatility \\
tang & Debt capacity/firm tangibility \\
tb & Tax income to book income \\
turn & Share turnover \\
zerotrade & Zero trading days \\

\end{longtable}

\end{document}